\documentclass[11pt,journal,onecolumn]{IEEEtran}
\usepackage[utf8x]{inputenc}
\usepackage{graphicx}
\usepackage[finalnew]{trackchanges}%,dsfont}
\usepackage{amsmath,url,amssymb,amsfonts,setspace,algorithm,algorithmic,epsfig,subfigure,epstopdf,color, float, wrapfig}
\usepackage[center,font = footnotesize]{caption}
\usepackage{amsthm, url, graphicx, caption, cite, bm, hyperref}
\usepackage{bm}

\long\def\comment#1{}

\newfont{\bbb}{msbm10 scaled 700}

\newfont{\bb}{msbm10 scaled 1100}

\newcommand{\xv}{{\bf x}}
\newcommand{\yv}{{\bf y}}

\newcommand{\Dm}{{\bf D}}

\newcommand{\Hm}{{\bf H}}
\newcommand{\Id}{{\bf I}}

\newcommand{\Um}{{\bf U}}
\newcommand{\Wm}{{\bf W}}

\newcommand{\Lam}{{\bf \Lambda}}

\newcommand{\Lcb}{{\bm {\mathcal L}}}

\newcommand{\mb}[1]{{\mathbf{#1}}}
\newcommand{\mc}[1]{{\mathcal{#1}}}
\newcommand{\mcb}[1]{{\bm{\mathcal{#1}}}}

\newtheorem{theorem}{Theorem}[section]

\title{Bilateral Filter: Graph Spectral Interpretation and Extensions}
\author{Akshay Gadde, Sunil K Narang and Antonio Ortega \\
Ming Hsieh Department of Electrical Engineering\\
         University of Southern California\\
         agadde@usc.edu, kumarsun@usc.edu, ortega@sipi.usc.edu
\thanks{This work was supported in part by NSF under grant CCF-1018977.}}
\begin{document}
% \ninept
\maketitle
\begin{abstract}
In this paper we study the bilateral filter proposed by Tomasi and Manduchi, as a spectral domain transform defined on a weighted graph. The nodes of this graph represent the pixels in the image and a graph signal defined on the nodes represents the intensity values. Edge weights in the graph correspond to the bilateral filter coefficients and hence are data adaptive. Spectrum of a graph is defined in terms of the eigenvalues and eigenvectors of the graph Laplacian matrix. We use this spectral interpretation to generalize the bilateral filter and propose more flexible and application specific spectral designs of bilateral-like filters. We show that these spectral filters can be implemented with $k$-iterative bilateral filtering operations and do not require expensive diagonalization of the Laplacian matrix.
\end{abstract}
%\vspace{-0.5cm}
\section{Introduction}
The bilateral filter (BF) proposed by Tomasi and Manduchi~\cite{Tomasi'98} has emerged as a powerful tool for adaptive processing of multidimensional data. Bilateral filtering smooths images while preserving edges, by taking the weighted average of the nearby pixels. The weights depend on both the spatial distance and photometric distance which provides local adaptivity to the given data. The bilateral filter and its variants are widely used in different applications such as denoising, edge preserving multi-scale decomposition, detail enhancement or reduction and segmentation etc.~\cite{ParisBF, BFdenoise,  multiscaleBF, durand, graphCutFilter}.
% another advantage of this framework is that it is general and can be easily extended to other filtering techniques.
Bilateral filtering was developed as an intuitive tool without theoretical justification. Since then connections between the BF and other well known filtering frameworks such as anisotropic diffusion, weighted least squares, Bayesian methods, kernel regression and non-local means have been explored~\cite{Elad'02, Barash, Milanfar'13, Singer, Baudes}.

The BF is data dependent and hence a non-linear and non-shift invariant filter. So, it does not have a spectral interpretation in the traditional frequency domain of images. However we would like to have spectral interpretation so that we can modify the global properties of the signal by changing its frequency components. To overcome this difficulty, we view the BF as a vertex domain transform on a graph with pixels as vertices, intensity values of each node as graph signal and filter coefficients as link weights that capture the similarity between nodes. This graphical views of the BF is used in~\cite{Peyre'08, Milanfar'13, Noel'12, graphCutFilter, Jian'09} but spectral design of filters for images using this graph has not been studied.

We can define spectral filters on these graphs where spectral response is calculated in terms of eigenvectors and eigenvalues of the graph Laplacian matrix~\cite{Hammond'11, Sunil'12, ShumanSPM, Moura}. This spectral interpretation captures the oscillatory behavior of the graph signal~\cite{nodalThm} and thus  allows us to extend of the concept of frequency to irregular domains. This has led to the design of frequency selective filtering operations on graphs similar to that in traditional signal processing. These graph spectral filters also have a vertex domain implementation.

In this paper, we interpret the BF as a $1$-hop localized transform on the aforementioned graph. Because the link weights in the graph are data adaptive, the problem of structure preserving filtering boils down to low pass spectral filtering on that graph. We show that the BF can be characterized by a spectral response corresponding to a linear spectral decay. We also calculate the spectral response of iterated BF. We extend this novel insight to build better and more general bilateral-like filters using the machinery to design graph based transforms with desired spectral response. Our design allows one to choose the spectral response of the filter depending on the application which offers more flexibility. We also give a theoretical justification for the design using the framework of regularization on graphs. We show that these spectral filters do not require computationally expensive diagonalization of the graph Laplacian matrix and then provide an efficient algorithm for implementing these filters using 
the BF as a building block. We examine the performance of the proposed filters in a few applications.

\section{Bilateral Filter as a Graph Based Transform}
Consider an input image $\mb{x}_{in}$ to the BF. The value at each position in the output image $\mb{x}_{out}$ is given by the weighted average of the pixels in $\mb{x}_{in}$.
\begin{equation}
\mb{x}_{out}[j] = \sum_i \frac{w_{ij}}{\sum_i w_{ij}} \mb{x}_{in}[i]
\label{eq:BF}
\end{equation}   
The weights $w_{ij}$ depend on both the euclidean and photometric distance between the pixels $\mb{x}_{in}[i]$ and $\mb{x}_{in}[j]$. Let $p_i$ denote the position of the pixel $i$. The weights are then given by
\begin{equation}
w_{ij} = \exp{\left(-\frac{\| p_i - p_j \|^2}{2\sigma_d^2}\right)} . \exp{\left(-\frac{(\mb{x}_{in}[i] - \mb{x}_{in}[j])^2 }{2\sigma_r^2}\right)}
\end{equation}
Spatial Gaussian weighting decreases the influence of distant pixels\footnote{Eq. (\ref{eq:BF}) shows that the averaging is done over all pixels. However in practice, one assumes non-zero weights only for the pixels which have $\|p_i - p_j \| \leq 2\sigma_s$~\cite{ParisBF}.} and intensity Gaussian weighting decreases the influence of pixels with different intensities. Intuition is that only similar nearby pixels should get averaged so that blurring of edges is avoided.

Now, consider an undirected graph $G = (\mc{V}, E)$ where the nodes $\mc{V} = \{1,2,\ldots,n\}$ are the pixels of the input image and the edges $E = \{(i,j,w_{ij})\}$ capture the similarity between two pixels as given by the BF weights (Figure~\ref{fig:BFgraph}). Image $\mb{x}_{in}$ can be considered as a signal defined on this graph $\mb{x}_{in}:\mc{V} \rightarrow \mathbb{R}$ where the signal value at each node equals the corresponding pixel intensity. 
\begin{wrapfigure}{r}{0.15\textwidth}
\centering
\includegraphics[width=0.6in]{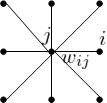}
\caption{The BF graph}
\label{fig:BFgraph}
\end{wrapfigure}
Adjacency matrix $\mb{W}$ of this graph is given by $\mb{W} = [w_{ij}]_{n\times n}$. Let $\mb{D}$ be the diagonal degree matrix where each diagonal element $\mb{D}_{jj} = \sum_i w_{ij}$. With this notation the filtering operation in \eqref{eq:BF} can be written as~\cite{Milanfar'13}
\begin{equation}
\mb{x}_{out} = \mb{D}^{-1} \mb{W} \mb{x}_{in}
\label{eq:BFmtx}
\end{equation}
It can be seen from \eqref{eq:BF} that the output at each node in the graph depends only on the nodes in its $1$-hop neighborhood. So, the BF is a $1$-hop localized graph based transform. Also, note that the BF includes the current pixel in the weighted average. So, the graph corresponding to the BF has a {\em self loop} i.e. an edge connecting each node to itself with weight $1$. 
%If we do not include the self loops in our graphical model, bilateral filtering  can be written as
%\begin{equation}
%\mb{x}_{out} = (\mb{I}+\mb{D})^{-1} (\mb{I}+\mb{W}) \mb{x}_{in}
%\end{equation}
% connection to learning. \delta I x + (1-\delta)W x. Effect of data adaptive graphs with and without self loops. include a figure.
Other filtering techniques such as Gaussian smoothing and non-local means can also be described using similar graphical models~\cite{Milanfar'13, Peyre'08}. 

\subsection{Graph Spectrum and Data Adaptivity of the BF}
Spectrum of a graph is defined in terms of the eigenvalues and eigenvectors of its Laplacian matrix. The combinatorial Laplacian matrix for the graph $G$ is defined $\mb{L} = \mb{D} - \mb{W}$. We use the normalized form of the Laplacian matrix given as $\mcb{L} = \mb{D}^{-1/2}\mb{L} \mb{D}^{-1/2}$. $\mcb{L}$ is a non-negative definite matrix~\cite{Chung'97}. As a result $\mcb{L}$ has an orthogonal set of eigenvectors $\mb{U} = \{\mb{u}_1,\ldots,\mb{u}_2\}$ with corresponding eigenvalues $\sigma(G) = \{\lambda_1,\ldots,\lambda_n\}$. So, $\mcb{L}$ can be diagonalized as
\begin{equation}
\mcb{L} = \mb{U}\mb{\Lambda}\mb{U}^t
\end{equation}
where $\mb{\Lambda} = diag\{\lambda_1,\ldots,\lambda_n\}$.

Similar to classical Fourier transform, the eigenvectors and eigenvalues of the Laplacian matrix $\mcb{L}$ provide a spectral interpretation 
of the graph signals. The eigenvalues $\{\lambda_1,\ldots,\lambda_n\}$ can be treated as graph frequencies, and are always situated in the interval $[0,2]$ on the real line. The eigenvectors of the Laplacian matrix demonstrate increasing oscillatory behavior as the magnitude of the graph frequency increases~\cite{nodalThm}. The {\em Graph Fourier Transform} (GFT) of a signal $\mb{x}$ is defined as its projection onto the eigenvectors of the graph, i.e., $\tilde{x}(\lambda_i) = \langle\mb{x},~\mb{u}_i\rangle$, or in matrix form $\tilde{\mb{x}} = \mb{U}^t\mb{x}$. The inverse GFT is given by $\mb{x} = \mb{U} \tilde{\mb{x}}$.

Figure~\ref{fig:BFgft} shows the fraction of total signal energy captured by the first $k$ spectral components of a graph signal corresponding to a $64\times 64$ image block. In this example we consider the following underlying graphs (1)~the BF graph (2)~the graph corresponding to Gaussian smoothing where the weights depend only on the geometric distance between two pixels.
%1. the BF graph with and without self loops and 2. the graph corresponding to Gaussian smoothing with and without self loops where the weights depend only %on the geometric distance between two pixels. First observation is that including self loops in the graph does not change the spectrum or the filter %output by much in either case. A more important point is that, 
We can see that due to the data adaptivity of the BF graph, most of the signal energy is captured in low frequency part of the BF graph spectrum in comparison with spectrum of Gaussian smoothing graph. So, the spectral basis of the BF graph offers better energy compaction for the given signal. 
\begin{figure}
\centering
\includegraphics[width=0.25\textwidth]{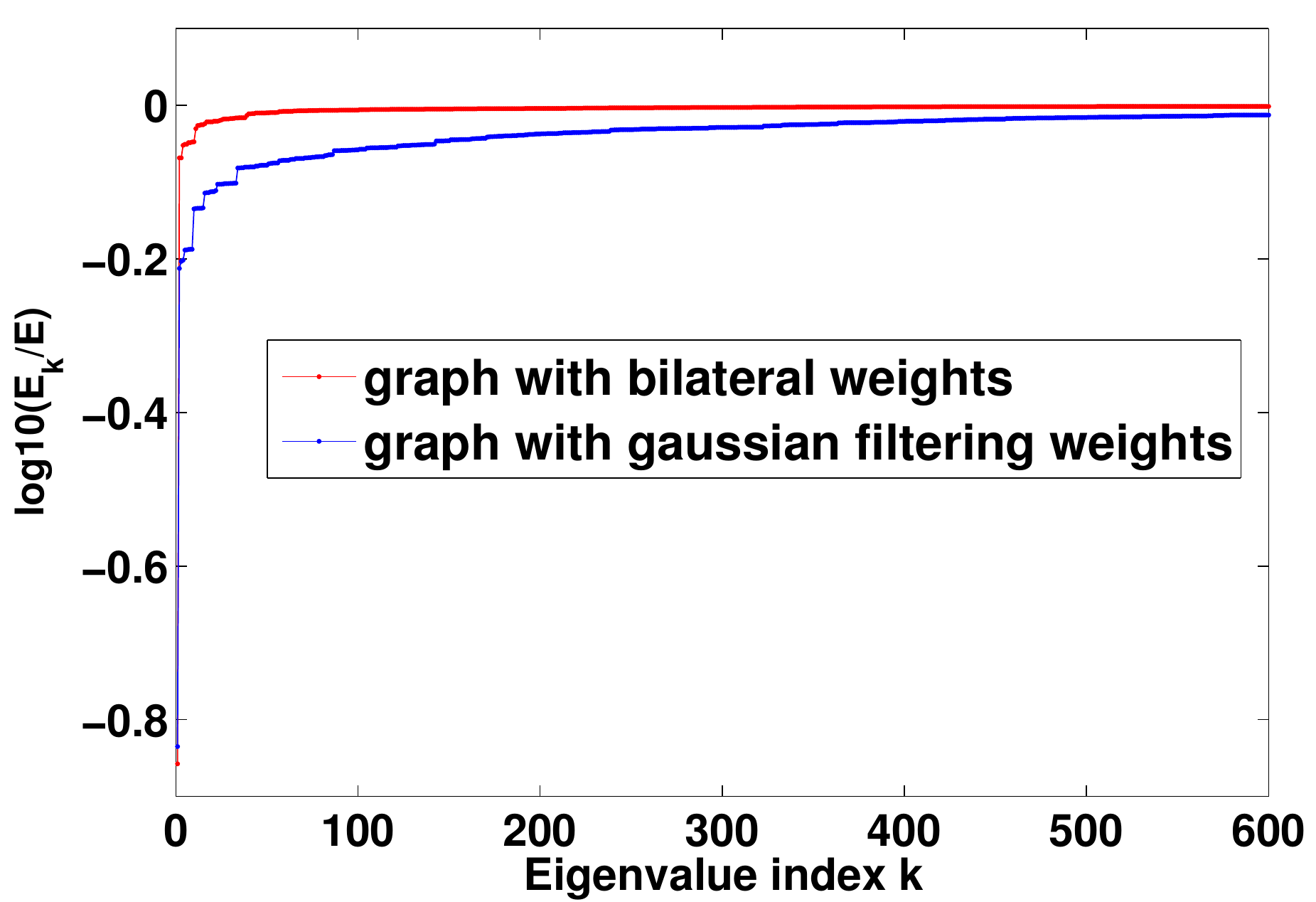}
\caption{$E_k$ is the fraction of total energy captured by the first $k$ spectral components.}
\label{fig:BFgft}
\end{figure}

\section{Spectral interpretation of the bilateral filter}	
Similar to conventional signal processing, {\em graph spectral filtering} is defined as
\begin{equation}
\tilde{x}_{out}(\lambda_i) = h(\lambda_i) \tilde{x}_{in}(\lambda_i)
\label{eq:gsf}
\end{equation}
$h(\lambda_i)$ is the spectral response of the filter according to which spectral components of an input signal are modulated. Using the definition of GFT and the diagonalized form of $\mcb{L}$, we can write graph spectral filtering in matrix notation as~\cite{ShumanSPM}
\begin{equation}
\mb{x}_{out} = \underbrace{\mb{U}}_{\substack{\text{Inverse}\\\text{GFT}}} \underbrace{h(\mb{\Lambda})}_{\substack{\text{Spectral}\\\text{response}}} \underbrace{\mb{U}^t \mb{x}_{in}}_{\text{GFT}} = h(\mcb{L}) \mb{x}_{in}
\label{eq:gtf}
\end{equation}

To exploit this framework of graph spectral filtering we rewrite the BF in \eqref{eq:BFmtx} as
\begin{align}
&\mb{x}_{out} = \mb{D}^{-1/2} \mb{D}^{-1/2} \mb{W} \mb{D}^{-1/2} \mb{D}^{1/2}\mb{x}_{in} \nonumber\\ 
\Rightarrow &\mb{D}^{1/2}\mb{x}_{out} = (\mb{I}-\mcb{L})\mb{D}^{1/2} \mb{x}_{in}
\label{eq:toBFsr}
\end{align}
From this equation, we can see that the BF is a graph transform, similar to the one in \eqref{eq:gtf}, operating on the normalized input signal $\mb{\hat{x}}_{in} =\mb{D}^{1/2}\mb{x}_{in}$ producing the normalized output $\mb{\hat{x}}_{out} = \mb{D}^{1/2}\mb{x}_{out}$.  This normalization allows us to define the BF in terms of the non-negative definite matrix $\mcb{L}$ and thus have a spectral interpretation. It also ensures that a constant signal when normalized, is an eigenvector of $\mcb{L}$ associated with zero eigenvalue~\cite{SunilBior}. Following \eqref{eq:gtf} we have,
\begin{equation}
\mb{\hat{x}}_{out} = \mb{U} (\mb{I}-\mb{\Lambda}) \mb{U}^t \mb{\hat{x}}_{in}
\label{eq:BFsr}
\end{equation}
This shows that the BF is a frequency selective graph transform with a spectral response $h_{BF}(\lambda_i) = 1-\lambda_i$ which corresponds to linear decay (See Figure \ref{fig:spRes}).
\begin{figure}
\centering
\includegraphics[width=0.25\textwidth]{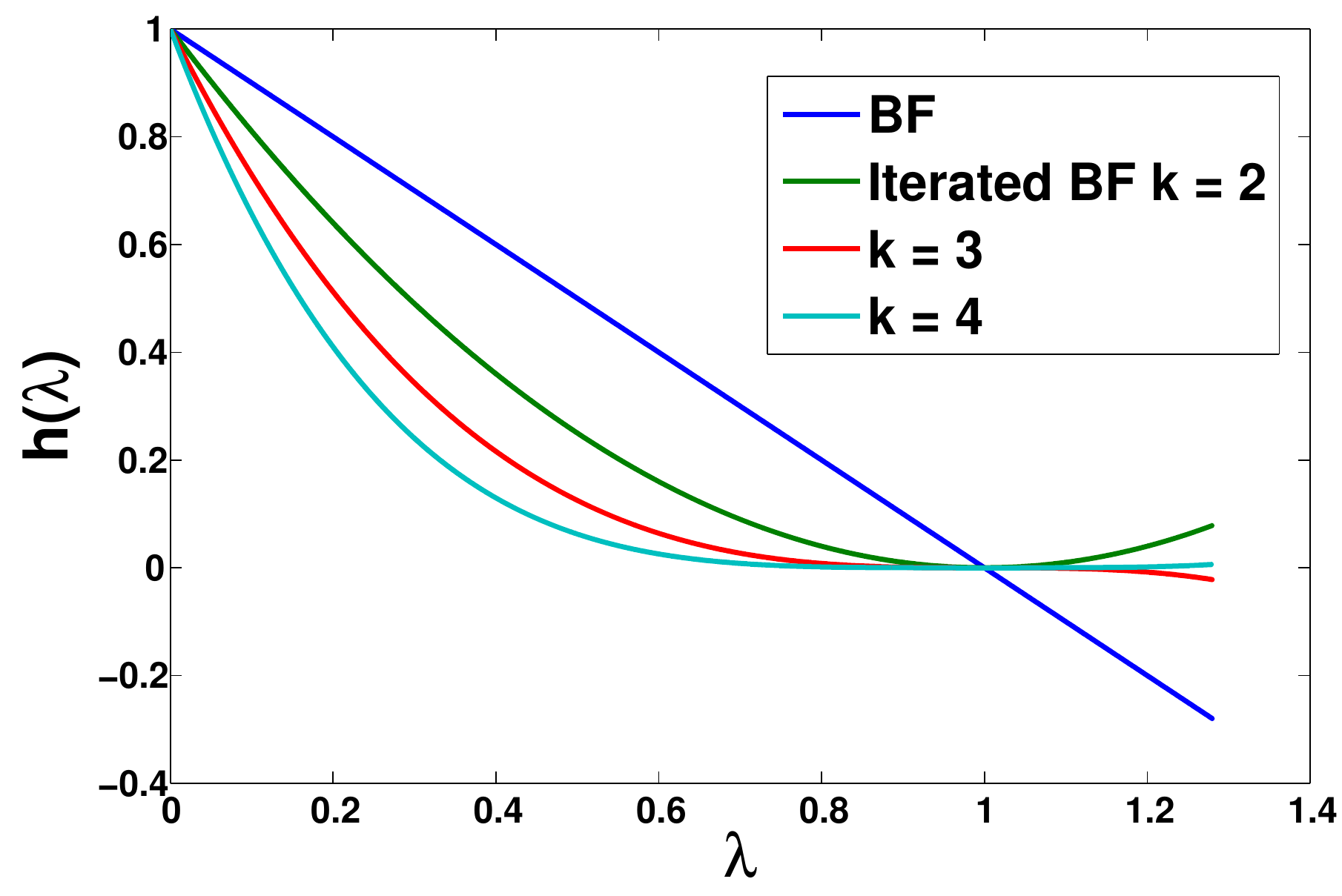}
\caption{Spectral responses of the BF and iterated BF.}
\label{fig:spRes}
\end{figure} 
The BF tries to preserve the low frequency components and attenuate the high frequency components. 

The BF is used iteratively in many applications. There are two ways to iterate the BF (1)~by changing the weights at each iteration using the result of previous iteration (2)~by using fixed weights at each iteration as calculated from the initial image. In the first method the BF graph changes in every iteration. So we cannot have a spectral interpretation. In the second method the graph remains fixed at every iteration. Here we consider the second method. Iterating preserves strong edges while removing weaker details. This type of effect is desirable for applications such as stylization~\cite{ParisBF}. The BF iterated $k$-times can be written in matrix notation as
\begin{equation}
\mb{x}_{out} = \left(\mb{D}^{-1} \mb{W}\right)^k \mb{x}_{in} = \left(\mb{I} - \mcb{L}_r \right)^k \mb{x}_{in}
\label{eq:BFiter}
\end{equation}
where $\mcb{L}_r = \mb{D}^{-1} \mb{L}$ is called the random walk Laplacian matrix. It can be shown that any graph transform $h(\mcb{L}_r)$ can be written in terms of $\mcb{L}$ as $h(\mcb{L}_r) = \mb{D}^{-1/2}h(\mcb{L})\mb{D}^{1/2}$~\cite[Proposition~2]{SunilBior}. Using this fact, we can rewrite \eqref{eq:BFiter} as
\begin{equation}
\mb{\hat{x}}_{out} = \mb{U} (\mb{I}-\mb{\Lambda})^k \mb{U}^t \mb{\hat{x}}_{in}
\label{eq:BFiterS}
\end{equation}
The spectral responses corresponding $k=2,3,4$ are shown in Figure \ref{fig:spRes}. The figure suggests that iterative application of the BF suppresses more of the high frequency component which is consistent with the observation. Equations \eqref{eq:BFsr} and \eqref{eq:BFiterS} give a different perspective to look at the bilateral filter. They hint at filter designs with better spectral responses that can be tailored to particular applications.

\section{Application specific spectral designs}
The BF and iterated BF have fixed spectral responses. But these responses may not be suitable for all applications. Below we discuss two applications (1)~image denoising and (2)~segmentation to illustrate design of more flexible spectral filters. 
%\vspace*{-0.2cm}
\paragraph*{Denoising.}
We consider the problem of image denoising with additive zero-mean white noise. 
\begin{equation}
\mb{y}[i] = \mb{x}[i] + \mb{e}[i]
\end{equation}
where $\mb{x}$ is the original image that we want to estimate, $\mb{y}$ is the observed noisy image and $\mb{e}$ is zero-mean white noise with variance $\sigma^2$. As explained before most of the signal energy lies in the low frequency part of the BF graph spectrum due to its data adaptivity. So most energy in the high frequency spectrum corresponds to noise. We can put this intuition in a more principled framework of regularization where the problem of denoising is equivalent to minimization of a penalty functional~\cite{Elad'02}. This penalty functional is composed of two terms. The first term is a fit measure and the second term is a data dependent smoothness constraint as captured  by regularization operators on the BF graph~\cite{Smola'03, Zhou'04}.
\begin{equation}
C(\mb{\hat{x}}) = \frac{1}{2}\|\mb{\hat{y}}-\mb{\hat{x}}\|^2 + \frac{\rho}{2}\|h_p(\mcb{L})\mb{\hat{x}}\|^2
\end{equation}
Note that we normalize $\mb{x}$ and $\mb{y}$ as in~\eqref{eq:toBFsr}. The regularization functional can be written in spectral domain as
\begin{equation}
\|h_p(\mcb{L})\mb{\hat{x}}\|^2 = \sum_{i=1}^n [h_p(\lambda_i)]^2[\tilde{\hat{x}}(\lambda_i)]^2
\end{equation}
$h_p(\lambda) \geq 0$ is chosen to be a non-decreasing function in $\lambda$ so that the high frequency components are penalized more strongly. Putting $\partial C/\partial \mb{\hat{x}} = 0$ we get the optimal $\mb{\hat{x}}$ as
\begin{align}
\mb{\hat{x}}_{opt} &= (\mb{I} + \rho h_p^2(\mcb{L}))^{-1}\mb{\hat{y}}\nonumber \\
				   &= \mb{U} (\mb{I} + \rho h_p^2(\mb{\Lambda}))^{-1} \mb{U}^t \mb{\hat{y}}
\end{align}
So, for a chosen regularization functional $h_p(\lambda)$, the spectral response of the denoising filter is given by
\begin{equation}
h_{opt}(\lambda) = \frac{1}{1 + \rho h_p^2(\lambda)}
\label{regfilt}
\end{equation}
Denoising filters suggested by the regularization framework are essentially low pass filters in the spectral domain of the graph (Figure~\ref{fig:denoise}). The BF has a relatively poor spectral decay profile due to which it blurs the textures in the denoising process. 
%\vspace*{-0.2cm}
\paragraph*{Image Segmentation.}
Shi and Malik~\cite{MalikNormCut} formulated the problem of image segmentation as a graph partitioning problem on a graph similar to the BF graph. To obtain an $m$-way partition of the graph, we only need to find the projection of the graph signal on the first few eigenvectors with the smallest eigenvalues. This projection gives a very coarse version of the signal which then can be used to perform image segmentation. So a suitable filter for this application should be a low pass filter in the graph spectral domain with a small cut off frequency and sharp transition band (Figure~\ref{fig:segment}). Iterative application of the BF also gives a coarse version of the image with details removed. But it favours the eigenvector with the smallest eigenvalue~\cite{graphCutFilter} which corresponds to the DC component of the image as seen from its spectral response. On the other hand, for graph partitioning we need the eigenvector with second smallest eigenvalue (and eigenvectors with larger eigenvalues for finer 
partitions).  

To summarize, different applications require filters having different spectral responses $h(\lambda)$. So, we would like to design more general bilateral-like filters with desired spectral response. These filters are of the form $h(\mcb{L}) = \mb{U} h(\mb{\Lambda}) \mb{U}^t$. A direct implementation of these filters requires diagonalization of $\mcb{L}$ which is of the order $O(N^3)$. For large graphs such as the one considered here, this is computationally very expensive. Fortunately, we can approximate any spectral response $h(\lambda)$ by a polynomial in $\lambda$. These polynomial spectral filters are $k$-hop localized on the graph where $k$ is the degree of the polynomial. This leads to an easy and efficient implementation scheme for these filters as explained in the next section.

\section{Polynomial approximation and fast implementation}
The spectral response of the iterative bilateral filter (Figure~\ref{fig:fastSF}(a)) given in (\ref{eq:BFiterS}) is a degree $k$ polynomial. This is a special case of a general class of real polynomials of degree $k$ given as
\begin{eqnarray}
 h(\Lam) = r_0\prod_{i=1}^k (\Id - r_i \Lam),
 \label{eq:poly_expansion_mat}
\end{eqnarray}
where the roots $r_i$ can be either real or complex conjugate pairs.
% The spectral response of 
% bilateral filter in (\ref{eq:BFsr}) can be written 
% in a general form as a one degree polynomial 
% $h_{BF}(\lambda) = (a - b\lambda)$ 
% of $\lambda$. 
% The corresponding transform $\Hm_{BF} = \Um h(\Lam)\Um^t$ 
% can be written as
% \begin{equation}
%  \Hm_{BF}^{a,b} = \Um(a\Id - b\Lam)\Um^t  = a\Id - b\Lcb,
%  \label{eq:general_BF}
% \end{equation}
% which is the same one degree polynomial 
% in term of $\Lcb$. 
This generalization allows $k+1$ degrees of freedom in choosing the spectral response of the filter.
% This can be further extended to a $K$ 
% degree polynomial spectral kernel given in (\ref{eq:poly_expansion_mat})
% having $K+1$ degrees of freedom. 
Further, the corresponding transform $\Hm$ in pixel domain is a  matrix polynomial of $\Lcb_r$ with the same roots as in (\ref{eq:poly_expansion_mat}) i.e.,
\begin{eqnarray}
 \Hm =  \Dm^{-1/2}\Um h(\Lam) \Um^t\Dm^{1/2} =  r_0\prod_{i = 1}^{k}(\Id - r_i\Lcb_r) .
 \label{eq:poly_expansion_L}
\end{eqnarray}
\begin{figure}[htb]
\begin{center}
 \subfigure[]{
   \includegraphics[width = 2.5in] {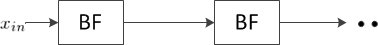}
 }%\\\vspace{-0.2cm}
\subfigure[]{
   \includegraphics[width = 2.5in] {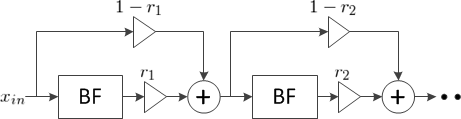}
 }
 %\vspace{-0.2cm}
\caption{(a)~Iterated BF (b)~Fast implementation of a bilateral-like filter with polynomial spectral response using BF as a building block}
\label{fig:fastSF}
\end{center}
\end{figure}
This leads to following result:
\begin{theorem}
Any graph filter on an image having a polynomial spectral response of degree $k$ can be implemented in the pixel domain as an iterative $k$ step bilateral filter operation. 
\end{theorem}
\begin{proof}
For $k = 1$, the filter $\Hm = \Id - r_1\Lcb = (1-r_1)\Id + r_1 \Dm^{-1}\Wm$. The output in this case is given in (\ref{eq:partial_BF}) whose convergence rate is $r_1$ times the convergence rate of the original bilateral filter $\Dm^{-1}\Wm$.
\begin{equation}
\yv = \Hm \xv = (1-r_1)\xv + r_1  \Dm^{-1}\Wm\xv,
\label{eq:partial_BF}
\end{equation}
For $k >1$, the transform $\Hm$ is a cascaded form of $k$ such bilateral filtering operations (See Figure~\ref{fig:fastSF}(b)). 
% The memory requirement for  computing $\Id - r_i\Lcb$ can be $\Oc(N)$ or $\Oc(2N)$ depending 
% upon if $r_i$ is real or complex respectively.
\end{proof}
Further, if $h(\lambda)$ is not a polynomial of $\lambda$, it can be approximated with a polynomial kernel which can then be implemented as a generalized iterative bilateral filtering operation. 
It has been shown in \cite{Hammond'11} 
that minimax polynomial approximation of any kernel $h(\lambda)$ not only minimizes the Chebychev norm (worst-case norm) 
of the error between kernel and its approximation, it also minimizes the upper-bound on the error $||H^{\text{exact}}- H^{\text{approx}}||$ 
between exact and approximated filters. In our experiments, we approximate any non-polynomial $h(\lambda)$
with the truncated Chebychev polynomials (which are a good
approximation of minimax polynomials).
%\vspace*{-0.2cm}
\section{Examples}
%\vspace*{-0.1cm}
\begin{figure}[htb]
\begin{center}
 \subfigure[]{
   \includegraphics[width = 2.5in] {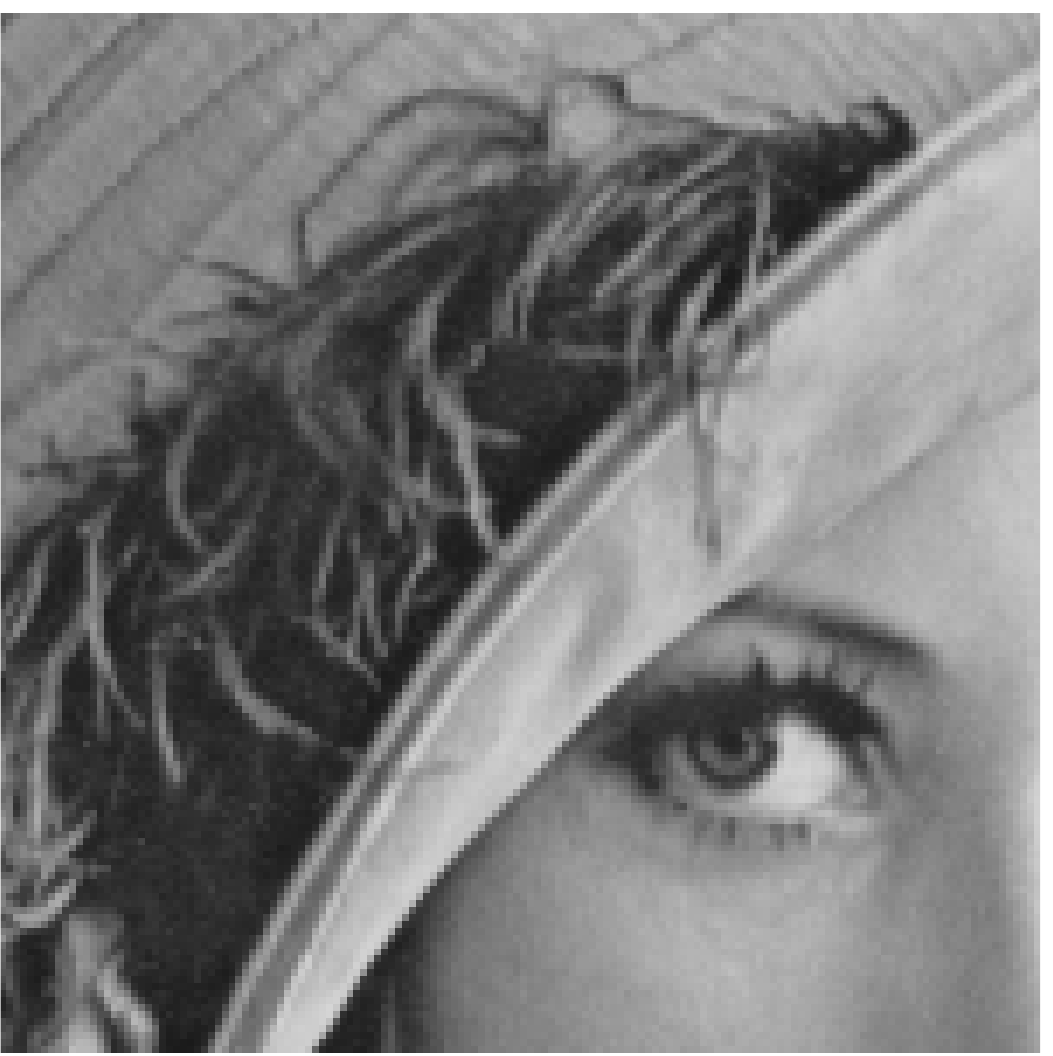}
 }
\subfigure[]{
   \includegraphics[width = 2.5in] {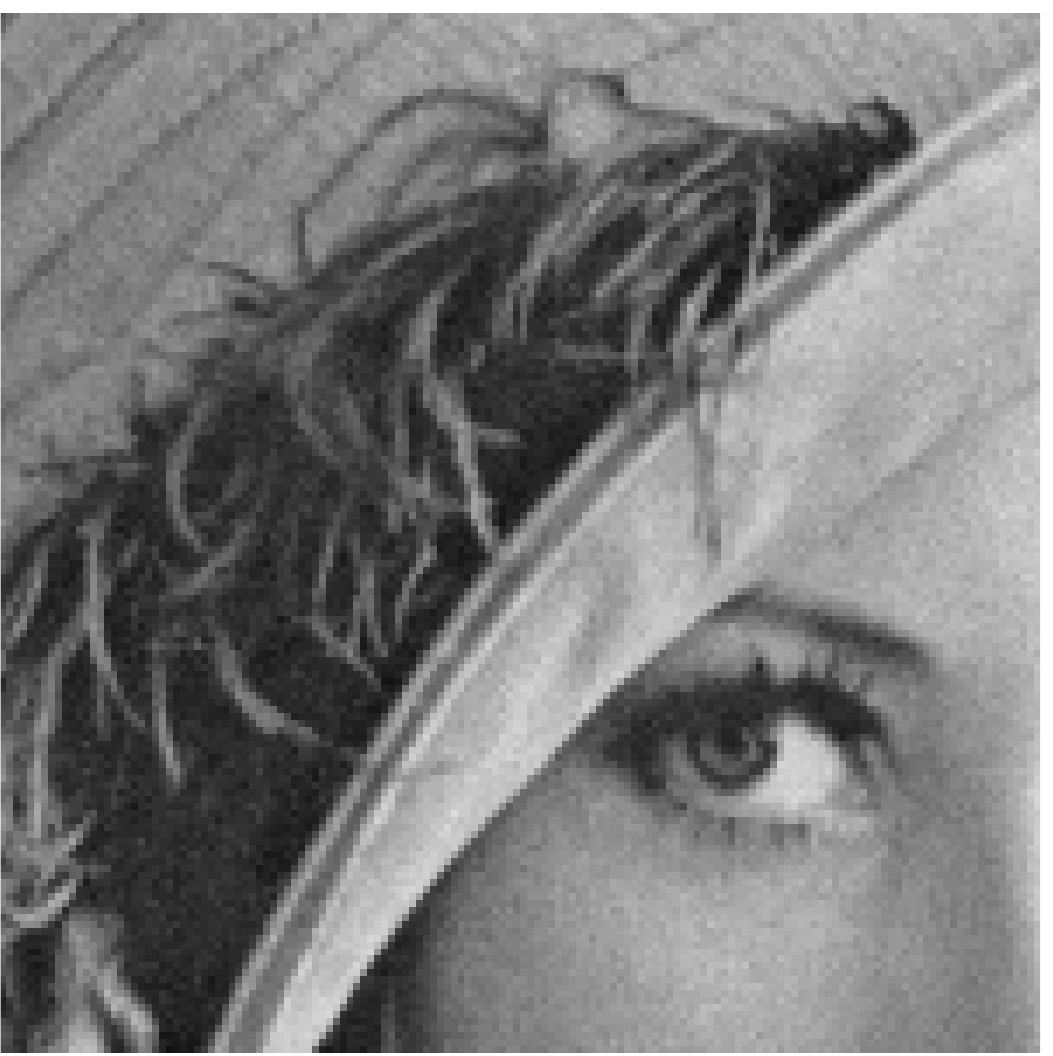}
 }\\%\vspace{-0.4cm}
 \subfigure[]{
   \includegraphics[width = 2.5in] {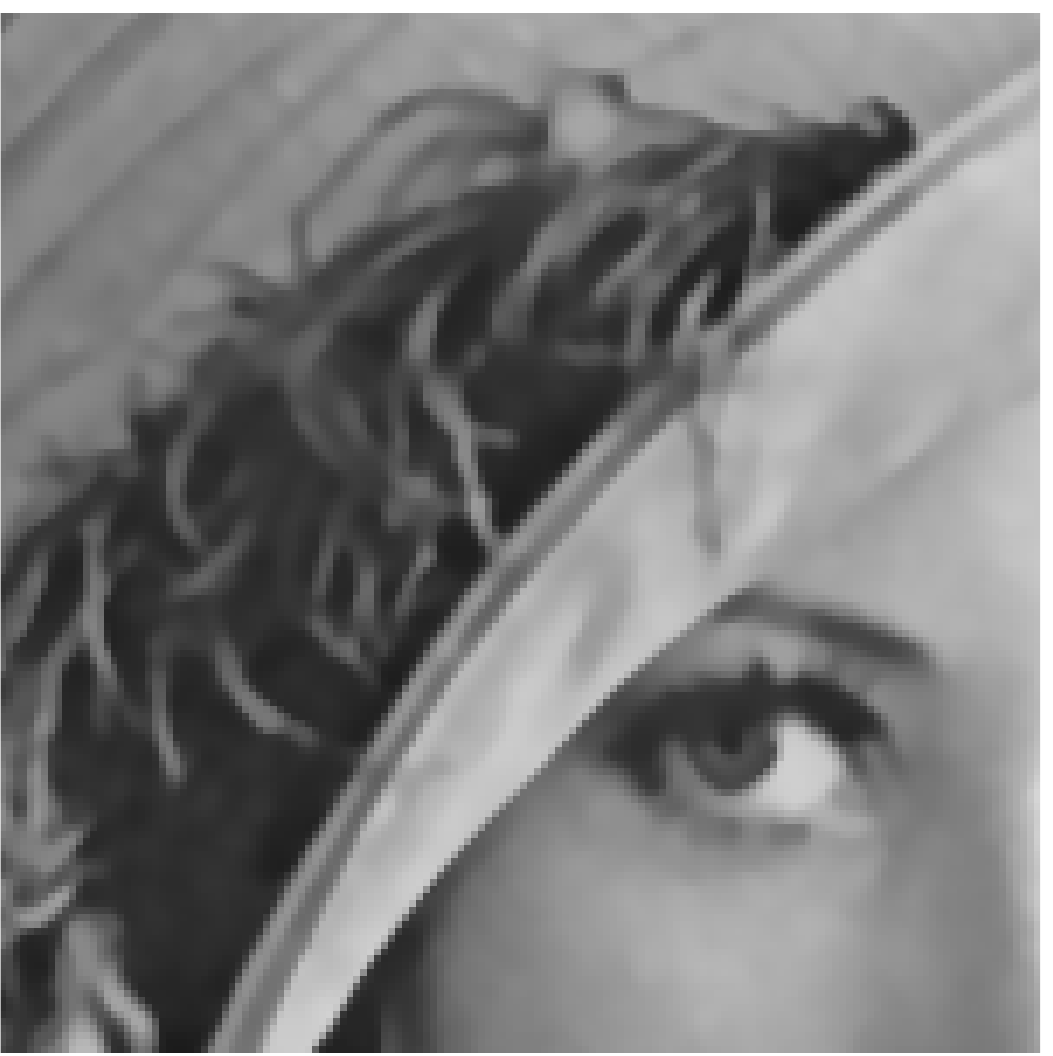}
 }
\subfigure[]{
   \includegraphics[width = 2.5in] {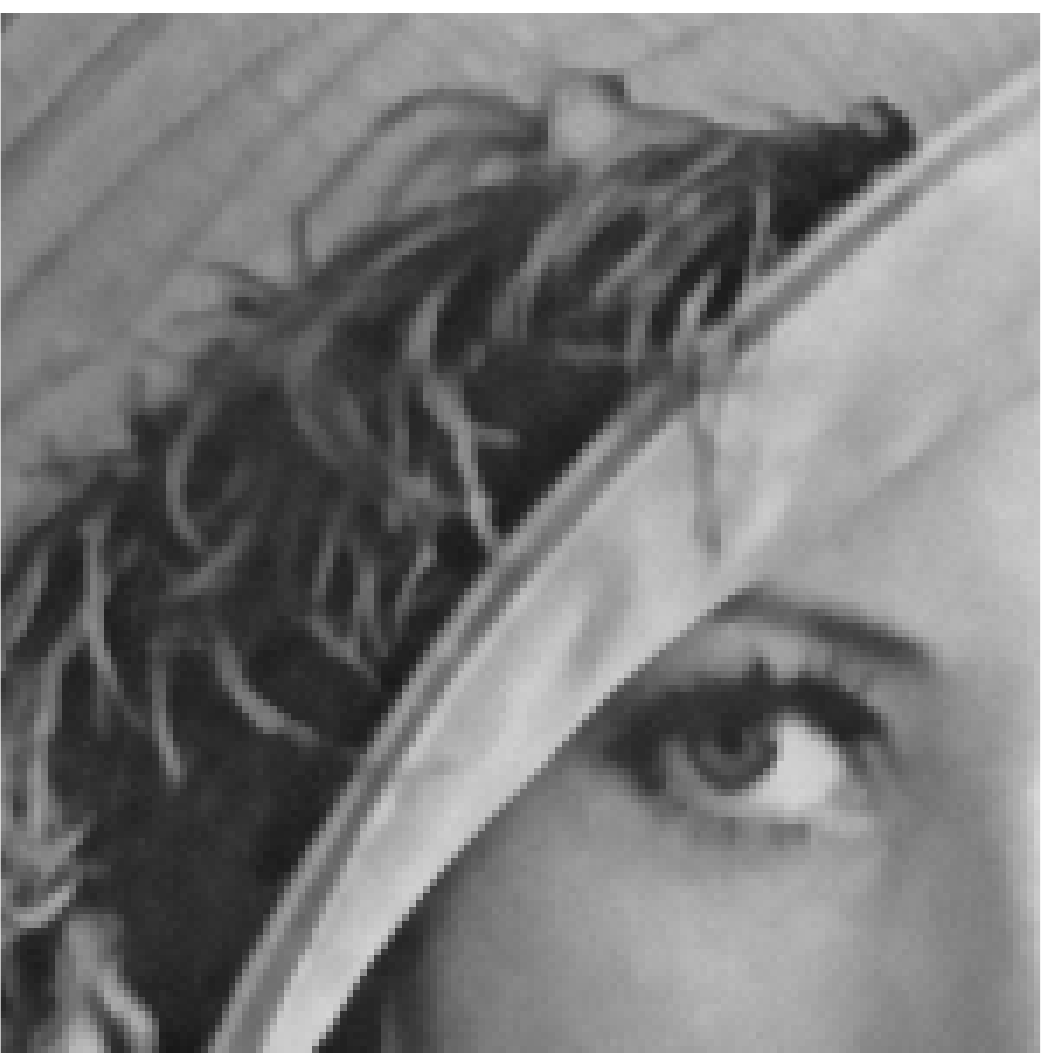}
 }\\%\vspace{-0.4cm}
 \subfigure[]{
   \includegraphics[width = 2.5in] {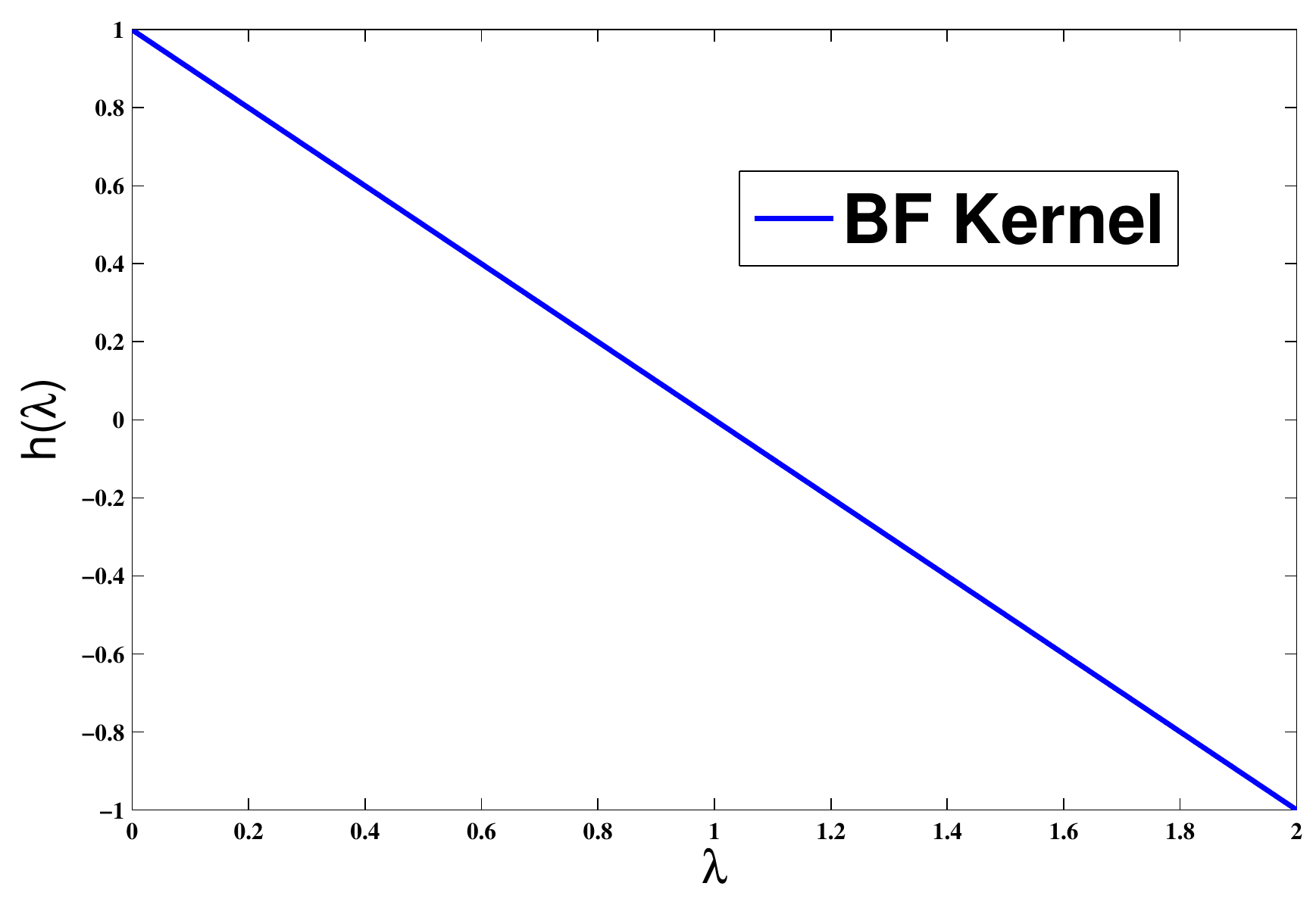}
 }
 \subfigure[]{
   \includegraphics[width = 2.5in] {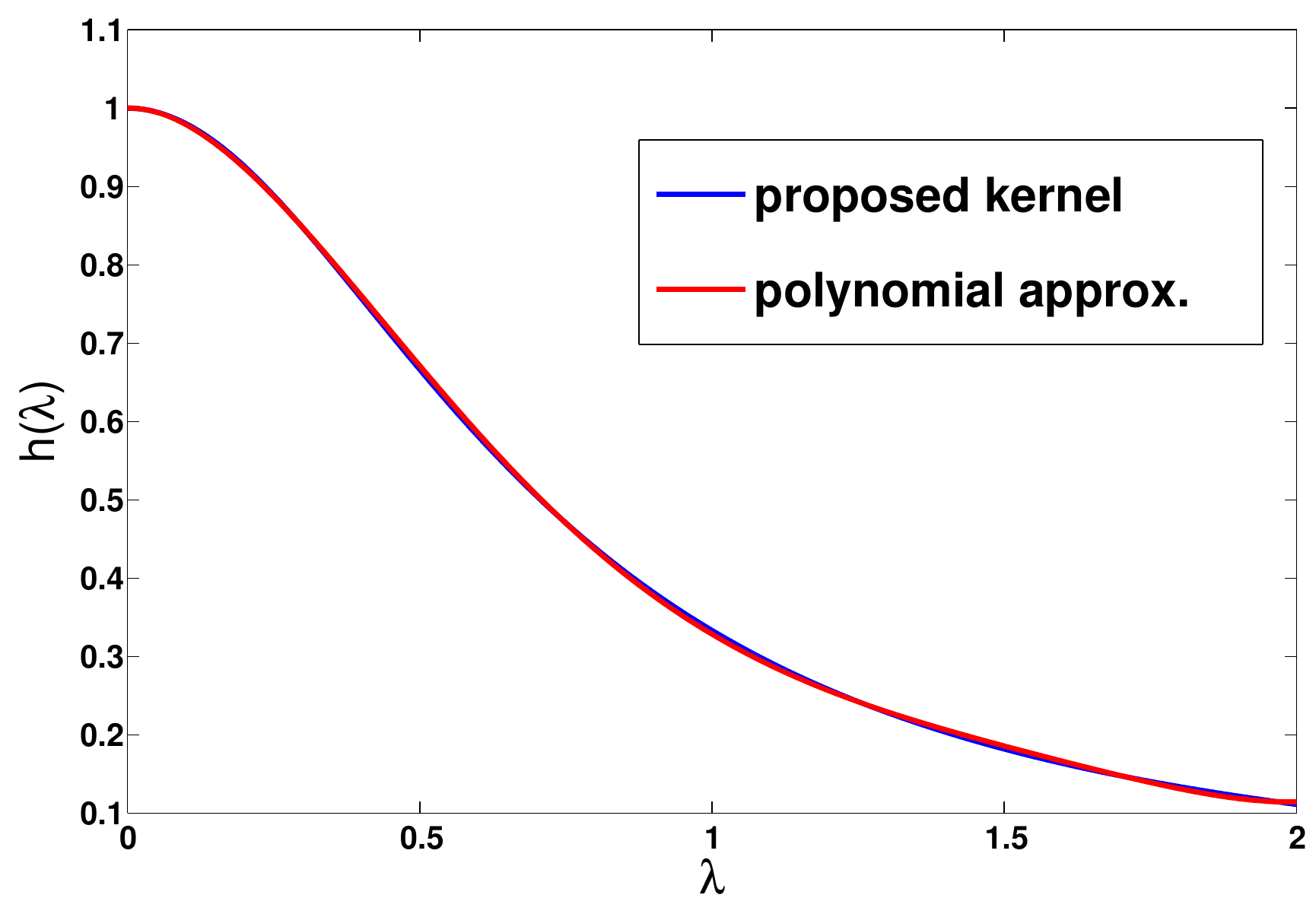}
 }
%\vspace{-0.4cm}
\caption{(a)~Original image (b)~Noisy image, SNR = 20 dB (c)~Output of the BF($\sigma_r = 0.035, \sigma_d = 2$), SNR = 20.65 dB (d)~Output of the proposed filter, SNR = 22.64 dB (e)~Spectral response of the BF (f)~Spectral response of proposed filter}
\label{fig:denoise}
\end{center}
\end{figure}

% \begin{figure}[htb]
% \begin{center}
%  \subfigure[]{
%    \includegraphics[width = 2.5in] {barb}
%  }
% \subfigure[]{
%    \includegraphics[width = 2.5in] {Orig}
%  }\\%\vspace{-0.4cm}
%  \subfigure[]{
%    \includegraphics[width = 2.5in] {BF_fixedW}
%  }
% \subfigure[]{
%    \includegraphics[width = 2.5in] {BF_fixed_poly}
%  }\\%\vspace{-0.4cm}
%  \subfigure[]{
%    \includegraphics[width = 2.5in] {segment_41}
%  }
%  \subfigure[]{
%    \includegraphics[width = 2.5in] {denoising_2}
%  }
% %\vspace{-0.4cm}
% \caption{(a)~Original image (b)~Noisy image, SNR = 20 dB (c)~Output of the BF($\sigma_r = 0.035, \sigma_d = 2$), SNR = 20.65 dB (d)~Output of the proposed filter, SNR = 22.64 dB (e)~Spectral response of the BF (f)~Spectral response of proposed filter}
% \label{fig:denoise2}
% \end{center}
% \end{figure}
We examine the performance of the proposed spectral design of bilateral-like filters in two applications. First, we consider 
the image denoising problem. 
We experiment with a spectral response obtained 
by the regularization framework 
in Section 4. We take the regularization functional 
$h_p(\lambda) = \lambda^2$ which suggests a denoising filter with 
$h(\lambda) = 1/(1+\lambda^2)$. We take its $5$ degree 
polynomial approximation. Figure~\ref{fig:denoise} shows 
the denoising results using this filter and the BF \footnote{The BF is not the best denoising filter available. 
We use it in our comparison to emphasize the qualitative differences in filtering results.}. 
It can be seen that the BF preserves edges, but it blurs the texture in the denoising process 
while the proposed denoising filter does a better job of preserving texture. This is also reflected in the SNR values. 
% Figure~\ref{denoise2}, shows effect of different iterated BF vs. the BF obtained by regularization framework in Section 4. Here we choose  
% a denoising filter with 
% $h(\lambda) = 1/(1+100*\lambda^2)$ (to further boost the smoothness condition). 
%\begin{figure*}
%\centering
%\includegraphics[width=0.8\textwidth]{denoise}
%\caption{(a) Original image ({\em lena} zoomed in) (d) Noisy image, SNR = 20 dB (b) Spectral response of the BF (c) Spectral response obtained by the regularization and its polynomial approximation (e) Output of the BF, SNR = 20.65 dB (f) Output of $h(\lambda)$ filter, SNR = 22.64 dB}
%\label{fig:denoise}
%\end{figure*}

Next, we consider an iterative application of the BF. Iterated BF removes minor details from the image while preserving prominent edges. This can be used as an effective preprocessing step in edge-detection and segmentation etc. As stated before, iterated BF favours the DC component which is not useful for segmentation. We use a low pass spectral kernel with small cut-off and sharp transition band so that the second (and a few higher) spectral components get at least as much weight as the DC component. We use a $20$ degree polynomial approximation of this kernel and compare it with $20$ iterations of the BF. Figure~\ref{fig:segment} shows that weak edges are blurred more with iterated BF (as expected from the spectral response) compared to the proposed filter.
\begin{figure}[htb]
\begin{center}
 \subfigure[]{
   \includegraphics[width = 2.5in] {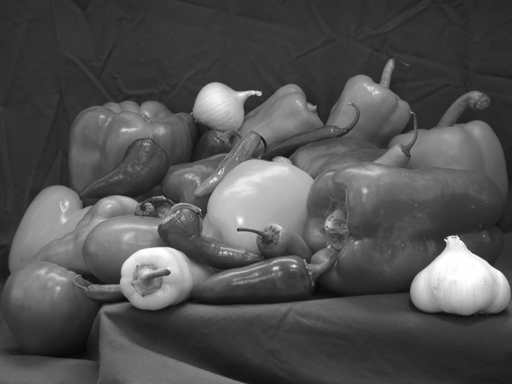}
 }
 \subfigure[]{
   \includegraphics[width = 2.5in] {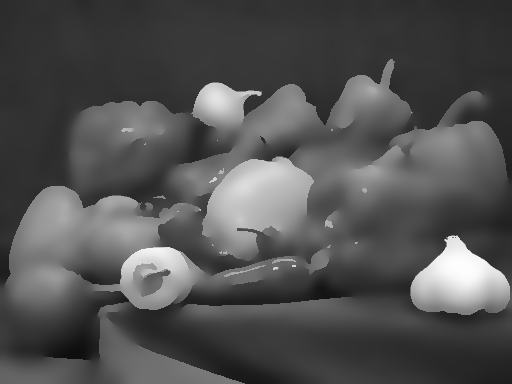}
 }\\%\vspace{-0.35cm}
\subfigure[]{
   \includegraphics[width = 2.5in] {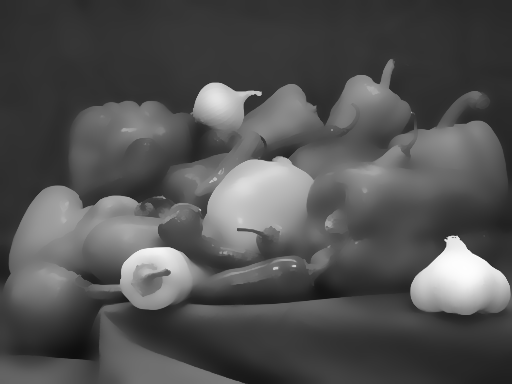}
 }
\subfigure[]{
   \includegraphics[width = 2.5in] {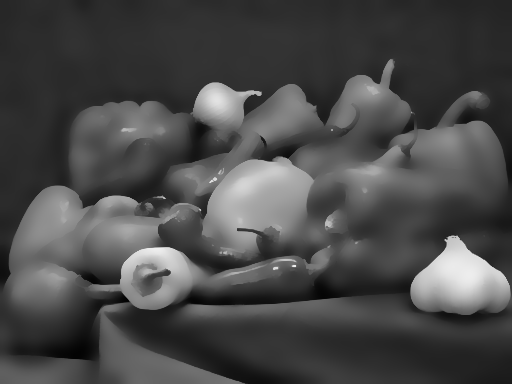}
 }\\%\vspace{-0.35cm}
\subfigure[]{
   \includegraphics[width = 2.5in] {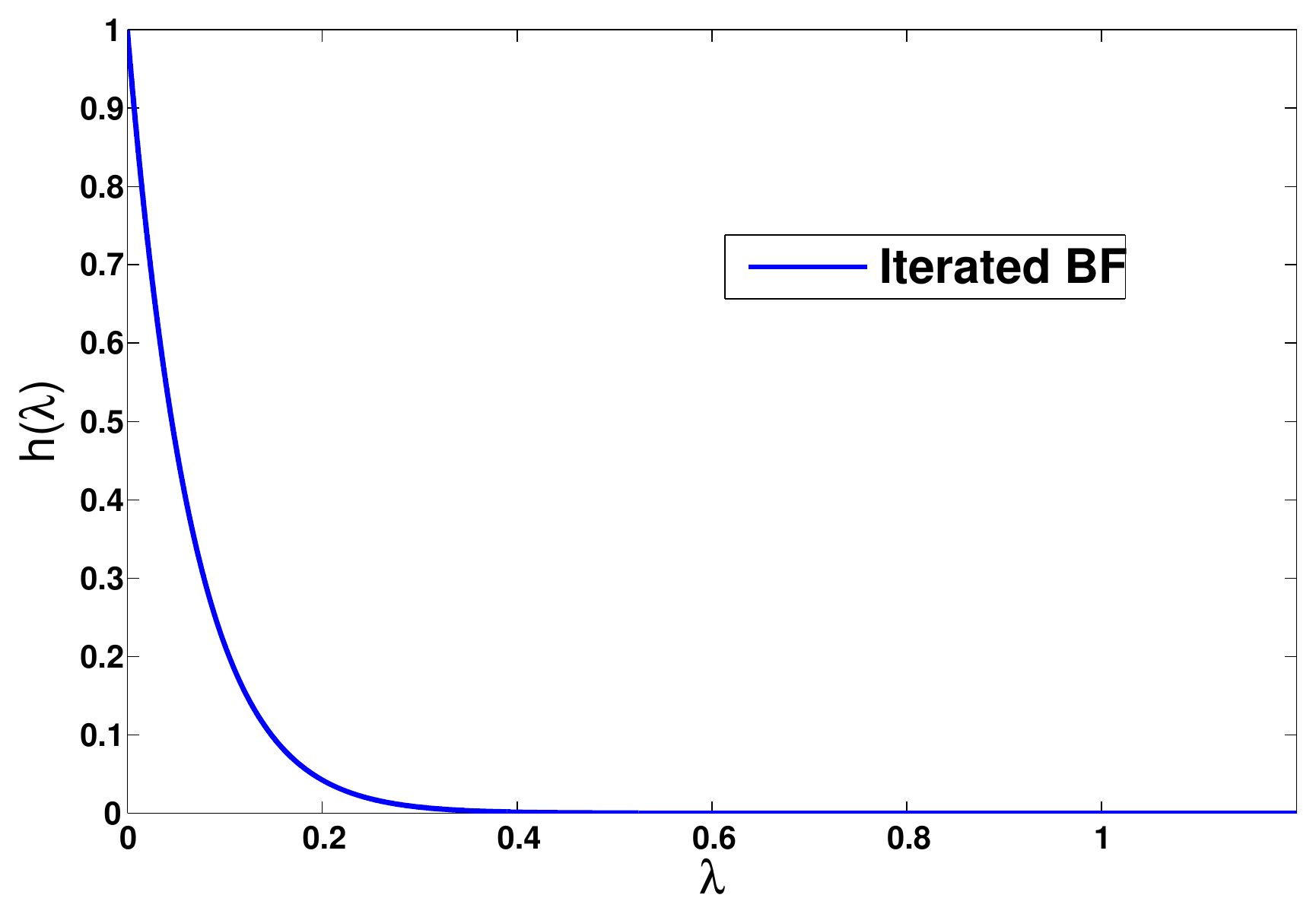}
 } 
\subfigure[]{
   \includegraphics[width = 2.5in] {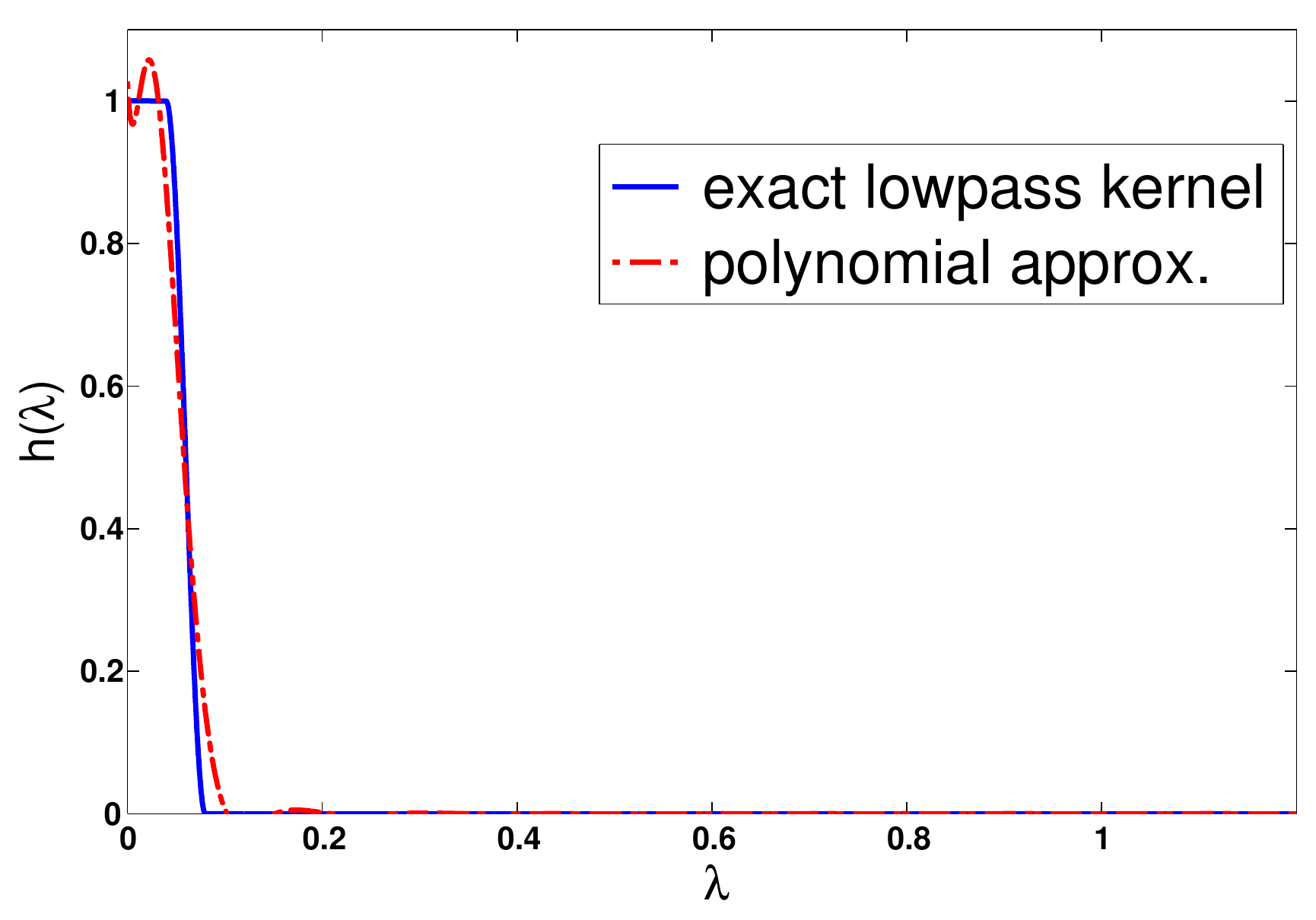}
 }
%\vspace{-0.35cm}
\caption{(a)~Original image ({\em peppers}) (b)~Output of 20 iteration of the BF with changing weights (c)~Output of 20 iteration of the BF with fixed weights($\sigma_r = 0.05, \sigma_d = 2$) (d)~output of the proposed spectral filter (e)~Iterated BF's spectral response (f)~Proposed spectral response} 
\label{fig:segment}
\end{center}
\end{figure}
%\begin{figure*}
%\centering
%\includegraphics[width=0.8\textwidth]{segment}
%\caption{(a)~Original image ({\em peppers}) (b)~Output of 30 iteration of the BF (c)~Spectral response of the iterated BF (d)~output of the proposed spectral filter  (e)~Corresponding Spectral response and its polynomial approximation} 
%\label{fig:segment}
%\end{figure*}
%\vspace*{-0.4cm}
\section{Conclusion}
%\vspace*{-0.1cm}
In this paper we explained the bilateral filter as a graph spectral filtering operation. With this novel perspective, we proposed a family of more flexible bilateral-like filters with desired spectral responses. We gave an easy implementation scheme for these filters. Their utility was motivated through few examples. An immediate interesting extension to this work would be to explore different spectral filters suitable for particular applications. Another topic of interest is the design of filter banks using these bilateral-like filters.

%\newpage
\bibliographystyle{IEEEbib}
\bibliography{BFrefs}

\begin{thebibliography}{10}

\bibitem{Tomasi'98}
C.~Tomasi and R.~Manduchi,
\newblock ``Bilateral filtering for gray and color images,''
\newblock in {\em Sixth International Conference on Computer Vision}, Jan 1998,
  pp. 839--846.

\bibitem{ParisBF}
S.~Paris, P.~Kornprobst, J.~Tumblin, and F.~Durand,
\newblock ``Bilateral filtering: Theory and applications,''
\newblock {\em Foundations and Trends in Computer Graphics and Vision}, vol. 4,
  no. 1, 2009.

\bibitem{BFdenoise}
Ming Zhang and B.~K. Gunturk,
\newblock ``Multiresolution bilateral filtering for image denoising,''
\newblock {\em Image Processing, IEEE Transactions on}, vol. 17, no. 12, pp.
  2324 --2333, Dec 2008.

\bibitem{multiscaleBF}
R.~Fattal, M.~Agrawala, and S.~Rusinkiewicz,
\newblock ``Multiscale shape and detail enhancement from multi-light image
  collections,''
\newblock {\em ACM Trans. Graph.}, vol. 26, no. 3, 2007.

\bibitem{durand}
F.~Durand and J.~Dorsey,
\newblock ``Fast bilateral filtering for the display of high-dynamic-range
  images,''
\newblock {\em ACM Trans. Graph.}, vol. 21, no. 3, pp. 257--266, July 2002.

\bibitem{graphCutFilter}
C.~Ye, Y.~Lin, M.~Song, C.~Chen, and D.~Jacobs,
\newblock ``Spectral graph cut from a filtering point of view,''
\newblock {\em arXiv:1205.4450v2 [cs.CV]}, 2012.

\bibitem{Elad'02}
M.~Elad,
\newblock ``On the origin of the bilateral filter and ways to improve it,''
\newblock {\em IEEE Transactions on Image Processing}, vol. 11, no. 10, pp.
  1141--1151, Oct 2002.

\bibitem{Barash}
D.~Barash,
\newblock ``Fundamental relationship between bilateral filtering, adaptive
  smoothing, and the nonlinear diffusion equation,''
\newblock {\em IEEE Transactions on Pattern Analysis and Machine Intelligence},
  vol. 24, no. 6, pp. 844--847, 2002.

\bibitem{Milanfar'13}
P.~Milanfar,
\newblock ``A tour of modern image filtering: New insights and methods, both
  practical and theoretical,''
\newblock {\em IEEE Signal Processing Magazine}, vol. 30, no. 1, pp. 106--128,
  Jan. 2013.

\bibitem{Singer}
A.~Singer, Y.~Shkolnisky, and B.~Nadler,
\newblock ``Diffusion interpretation of nonlocal neighborhood filters for
  signal denoising,''
\newblock {\em SIAM Journal on Imaging Sciences}, vol. 2, no. 1, pp. 118--139,
  2009.

\bibitem{Baudes}
A.~Buades, B.~Coll, and J.M. Morel,
\newblock ``A review of image denoising algorithms, with a new one,''
\newblock {\em Multiscale Modeling \& Simulation}, vol. 4, no. 2, pp. 490--530,
  2005.

\bibitem{Peyre'08}
G.~Peyr\'e,
\newblock ``Image processing with nonlocal spectral bases,''
\newblock {\em SIAM Journal on Multiscale Modeling and Simulation}, vol. 7, no.
  2, pp. 703--730, 2008.

\bibitem{Noel'12}
G.~Noel, K.~Djouani, B.~Van Wyk, and Y.~Hamam,
\newblock ``Bilateral mesh filtering,''
\newblock {\em Pattern Recognition Letters}, vol. 33, no. 9, pp. 1101--1107,
  2012.

\bibitem{Jian'09}
C.~Jian, K.~Inoue, K.~Hara, and K.~Urahama,
\newblock ``Fixed-coefficient iterative bilateral filters for graph-based image
  processing,''
\newblock in {\em Advances in Image and Video Technology}, vol. 5414 of {\em
  Lecture Notes in Computer Science}, pp. 473--484. Springer, 2009.

\bibitem{Hammond'11}
D.~Hammond, P.~Vandergheynst, and R.~Gribonval,
\newblock ``Wavelets on graphs via spectral graph theory,''
\newblock {\em Applied and Computational Harmonic Analysis}, vol. 30, no. 2,
  pp. 129 -- 150, 2011.

\bibitem{Sunil'12}
S.K. Narang and A.~Ortega,
\newblock ``Perfect reconstruction two-channel wavelet filter banks for graph
  structured data,''
\newblock {\em IEEE Transactions on Signal Processing}, vol. 60, no. 6, pp.
  2786--2799, June 2012.

\bibitem{ShumanSPM}
D.~Shuman, S.~Narang, P.~Frossard, A.~Ortega, and P.~Vandergheynst,
\newblock ``Signal processing on graphs: Extending high-dimensional data
  analysis to networks and other irregular data domains,''
\newblock {\em to appear in IEEE Signal Processing Magazine}, 2013.

\bibitem{Moura}
A.~Sandryhaila and J.~M.~F. Moura,
\newblock ``Discrete signal processing on graphs,''
\newblock {\em arXiv:1210.4752v2 [cs.SI]}, 2012.

\bibitem{nodalThm}
E.B. Davies, G.M.L. Gladwell, J.~Leydold, and P.F. Stadler,
\newblock ``Discrete nodal domain theorems,''
\newblock {\em Linear Algebra and its Applications}, vol. 336, no. 1, pp.
  51--60, 2001.

\bibitem{Chung'97}
F.R.K. Chung,
\newblock {\em Spectral graph theory}, vol.~92,
\newblock CBMS Regional Conference Series in Mathematics, AMS, 1997.

\bibitem{SunilBior}
S.K. Narang and A.~Ortega,
\newblock ``Compact support biorthogonal wavelet filterbanks for arbitrary
  undirected graphs,''
\newblock {\em arXiv:1210.8129 [cs.IT]}, 2012.

\bibitem{Smola'03}
A.~Smola and R.~Kondor,
\newblock ``Kernels and regularization on graphs,''
\newblock in {\em Proceedings of the Annual Conference on Computational
  Learning Theory}. 2003, Lecture Notes in Computer Science, Springer.

\bibitem{Zhou'04}
D.~Zhou and B.~Sch{\"o}lkopf,
\newblock ``A regularization framework for learning from graph data,''
\newblock in {\em ICML Workshop on Statistical Relational Learning}, 2004, pp.
  132--137.

\bibitem{MalikNormCut}
J.~Shi and J.~Malik,
\newblock ``Normalized cuts and image segmentation,''
\newblock {\em IEEE Transactions on Pattern Analysis and Machine Intelligence},
  vol. 22, no. 8, pp. 888--905, Aug 2000.

\end{thebibliography}
\end{document}